\documentclass{article}

\usepackage{graphicx}
\usepackage{amsmath,amsthm,amssymb,bm}

\usepackage[small]{caption}
\usepackage{subcaption}

\usepackage{pifont}
 
\usepackage{xcolor}

\usepackage{hyperref}

\usepackage{url}
\usepackage[a4paper]{geometry}
\usepackage{authblk}

\graphicspath{{./}{./figures/}}

\def\var{\mathrm{var}}
\def\CR{\mathrm{CR}}
\def\calP{\mathcal{P}}

\def\var{\mathrm{var}}

\def\calM{\mathcal{M}}

\def\calX{\mathcal{X}}

\def\calD{\mathcal{D}}

\def\bbR{\mathbb{R}}

\def\st{\ :\ }

\def\HG{\mathrm{HG}}

\def\NH{\mathrm{NH}}  

\def\eqdef{{:=}}

\def\min{\mathrm{min}}
\def\max{\mathrm{max}}

\def\rhofd{\rho_{\mathrm{FD}}}
\def\rhohg{\rho_{\mathrm{HG}}}

\def\rhol1{\rho_{\mathrm{L1}}}

\renewcommand{\Re}{\mathbb{R}}

\newtheorem{theorem}{Theorem}
\newtheorem{lemma}{Lemma}
\newtheorem{property}{Property}

\title{Non-linear Embeddings in Hilbert Simplex Geometry\footnote{Presented at the 2nd Annual Workshop on Topology, Algebra, and Geometry in Machine Learning (TAG-ML), ICML 2023.}}

\author[$\star$]{Frank Nielsen}
\affil[$\star$]{Sony Computer Science Laboratories Inc.}
\affil[$\star$]{Japan}
\affil[$\star$]{ORCID:~0000-0001-5728-0726}
\affil[$\star$]{{\small\tt Frank.Nielsen@acm.org}\vspace{.5em}}
\author[$\dagger$]{{Ke Sun}}
\affil[$\dagger$]{CSIRO Data61, Australia}
\affil[$\dagger$]{The Australian National University}
\affil[$\dagger$]{ORCID:~0000-0001-6263-7355}
\affil[$\dagger$]{{\small\tt sunk@ieee.org}}
\date{}

\begin{document}
\maketitle

\begin{abstract}
A key technique of machine learning is to embed discrete weighted graphs into continuous spaces for further downstream analysis.
Embedding discrete hierarchical structures in hyperbolic geometry has proven very successful since it was shown that any weighted tree can be embedded in that geometry with arbitrary low distortion.
Various optimization methods for hyperbolic embeddings based on common  models of hyperbolic geometry have been studied.
In this paper, we consider the Hilbert geometry of the standard simplex which is isometric to a vector space equipped with a symmetric polytope norm.
We study the representation power of this Hilbert simplex geometry by embedding distance matrices of graphs using a fast differentiable approximation of the Hilbert metric distance.
Our findings demonstrate that Hilbert simplex geometry is competitive to alternative geometries such as the Minkowski hyperboloid model of hyperbolic geometry or the Euclidean geometry for embedding tasks while being fast and numerically robust.
\end{abstract}

\section{Introduction}

Since Sarkar~\cite{Sarkar-2011} proved that any weighted tree graph can be embedded as a Delaunay subgraph of points in hyperbolic geometry   embedding nodes with arbitrary small distortions,
hyperbolic embeddings have become widely popular in machine learning~\cite{HE-2018} and computer vision~\cite{HIE-2020} to represent various hierarchical structures~\cite{HE-Structure-2021}.
Various models~\cite{HVD-2014} of hyperbolic geometry have been considered from the viewpoint of time efficiency and numerical stability~\cite{HE-2020} (e.g., Poincar\'e model~\cite{HE-Poincare-2017}, Minkowski hyperboloid model~\cite{spacetime,HE-Hyperboloid-2020}, Klein model~\cite{HME-Klein2020}, Lorentz model~\cite{LorentzHE-2018}, etc.)
 and extensions to symmetric matrix spaces~\cite{SymmetricSpace-2021} have also been considered recently.

 In this work, we consider Hilbert geometry~\cite{HandbookHilbert-2014} which can be seen as a generalization of Klein model of hyperbolic geometry where the unit ball domain is replaced by an arbitrary open bounded convex domain $\Omega$.
 When the boundary $\partial\Omega$ is smooth,  Hilbert geometry is of hyperbolic type (e.g., Cayley-Klein geometry~\cite{Richter-2011,shao2017machine} when $\Omega$ is an ellipsoid). When the domain is a bounded polytope, Hilbert geometry is  bilipschitz quasi-isometric to a normed vector space~\cite{HGPolytope-2014}, and isometric to a vector space with a polytope norm only when $\Omega$ is a simplex~\cite{HilbertHarpe-1991}.
 It is thus interesting to consider Hilbert  simplex geometry for embeddings and compare its representation performance to hyperbolic embeddings.
Hilbert simplex geometry has been used in machine learning for clustering histograms or correlation matrices~\cite{nielsen2019clustering}.

The paper is organized as follows:
 In \S\ref{sec:HSG}, we present the Hilbert distance as a symmetrization of the oriented Funk weak distances, describe some properties of the Hilbert simplex distance, and illustrate qualitatively the ball shapes for the Funk and Hilbert distances~\cite{HG-SoCG-2017}.
 We first consider Hilbert simplex linear embeddings
 and prove that Hilbert simplex distance is a non-separable monotone distance (Theorem~\ref{thm:FunkHilbertMonotone} in \S\ref{sec:monotone}).
 Monotonicity of distances is an essential property: It states that the distance can only decrease by linear embeddings into  smaller dimensional spaces.
 In information geometry, separable monotone divergences are exactly the class of $f$-divergences~\cite{amari2016information}.
 Aitchison non-separable distance used in compositional data analysis was  proven monotone~\cite{CODA-2021} only recently.
We explain a connection between Aitchison distance and Hilbert distance by using the variation semi-norm in \S\ref{sec:isometryVS}.
 Section~\ref{sec:exp} presents our experiments which demonstrates that in practice  Funk and Hilbert non-linear embeddings outperforms or is competitive compared to various other distances (namely, Euclidean/Aitchison distance, $\ell_1$-distance, hyperbolic Poincar\'e distance) while being fast and robust to compute.
Section~\ref{sec:concl} concludes this work.

\section{Hilbert simplex geometry}\label{sec:HSG}

\subsection{Definition}

Let $\Omega$ be any open bounded convex set of $\Re^d$.
The Hilbert distance~\cite{Hilbert-1895,BH-2014,HandbookHilbert-2014} $\rhohg^{\Omega}(p,q)$ between two points $p,q\in\Omega$ induced by $\Omega$
is defined as the symmetrization of the Funk distance $\rhofd^\Omega(p,q)$.
The Funk distance~\cite{FunkHilbert-2014} is defined by
$$
\rhofd^\Omega(p,q) :=\left\{
\begin{array}{ll}
 \log\left( \frac{\|p-\bar{q}\|}{\|q-\bar{q}\|} \right)\geq 0, & p\not=q,\\
0 & p=q.
\end{array}
\right.
$$
where $\bar{q}$ denotes the intersection of the affine ray $R(p,q)$
emanating from $p$ and passing through $q$ with the domain boundary $\partial\Omega$.
See Figure~\ref{fig:FunkSimplex} for an illustration.
The Funk distance is a weak metric distance since it satisfies the triangle inequality of metric distances but is an asymmetric dissimilarity measure: $\rhofd^{\Omega}(p,q)\not=\rhofd^{\Omega}(q,p)$.

\begin{figure}[t!]
\centering
\includegraphics[width=.5\columnwidth]{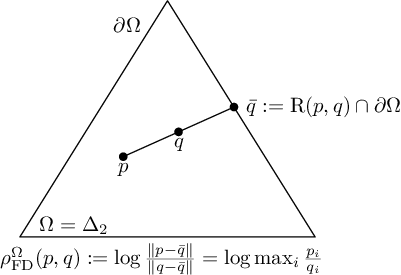}
\caption{Funk distance defined in the open standard simplex.\label{fig:FunkSimplex}}
\end{figure}

Thus the Hilbert distance $\rhohg^\Omega(p,q)$ between any two points $p,q\in\Omega$ is:
\begin{eqnarray*}
\rhohg^\Omega(p,q) &:=& \rhofd^\Omega(p,q)+\rhofd^\Omega(q,p),\\
&=& \left\{
\begin{array}{ll}
\log
\frac{\|p-\bar{q}\| \, \|q-\bar{p}\|}
{\|p-\bar{p}\| \, \|q-\bar{q}\|}
, & p\not=q,\\
0 & p=q.
\end{array}
\right.
\end{eqnarray*}
where $\bar{p}$ and $\bar{q}$ are the two intersection points of the line $(pq)$ with $\partial\Omega$, and the four collinear points are arranged in the order $\bar{p},p,q,\bar{q}$.
The $d$-dimensional Hilbert distance can also be interpreted as a 1D Hilbert distance induced by the 1D interval domain $\Omega_{pq}:=\Omega\cap(pq)$:
$$
\rhohg^\Omega(p,q) = \rhohg^{\Omega_{pq}}(p,q).
$$
This highlights that the quantity
$\frac{\|p-\bar{q}\| \, \|q-\bar{p}\|}%
{\|p-\bar{p}\| \, \|q-\bar{q}\|}$
does not depend on the chosen norm $\|\cdot\|$ because we can consider the absolute value $|\cdot|$ on the domain $\Omega_{pq}$.
For any $x$, $y\in\Re$,
$\|x\| = c |x|$ and $\|x-y\| = c |x-y|$ where $c>0$ is a constant.
Thus we can express the Hilbert distance as the logarithm of the cross-ratio:
$$
\rhohg^\Omega(p,q) = \left\{
\begin{array}{ll}
\log \CR(\bar{p},p;q,\bar{q}), & p\neq{}q,\\
0 & p=q,
\end{array}
\right.
$$
where $\CR(\bar{p},p;q,\bar{q}):=\frac{\|p-\bar{q}\| \, \|q-\bar{p}\|}{\|p-\bar{p}\| \, \|q-\bar{q}\|}$ denotes the cross-ratio.
The Hilbert distance is a metric distance, and
it follows from the properties of the cross-ratio~\cite{Richter-2011} that straight lines are geodesics in Hilbert geometry:
$$
\forall r\in[pq],\quad \rhohg^\Omega(p,q)=\rhohg^\Omega(p,r)+\rhohg^\Omega(r,q),
$$
where $[pq]$ is the closed line segment connecting $p$ and $q$.

Another property is that the Hilbert distance is invariant under homographies~\cite{hartley2003multiple} $H$ also called collineations (projective invariance):
$$
\rhohg^{H\Omega}(Hp,Hq)=\rhohg^{\Omega}(p,q),
$$
where $H\Omega:=\{Hp\ :\ p\in\Omega\}$.
The Hilbert geometry of the complex Siegel ball generalizing the Klein ball has been studied in~\cite{nielsen2020siegel}.

\subsection{Hilbert simplex distance}

We shall consider $\Omega=\Delta_d$, the open $(d-1)$-dimensional simplex:
$$
\Delta_d:=\left\{(x_1,\ldots,x_d)\in\bbR_{++}^d\ :\ \sum_{i=1}^d  x_i=1\right\},
$$
where $\Re_{++}:=(0,\infty)$.

In that case, we write $\rhofd(p,q):=\rhofd^{\Delta_d}(p,q)$, and we have
\begin{equation}
\rhofd(p,q)=\log \max_{i\in\{1,\ldots,d\}} \frac{p_i}{q_i}.
\end{equation}
Thus the Hilbert distance induced by the standard simplex is
\begin{eqnarray}
 \rhohg(p,q) &=& \rhofd(p,q)+\rhofd(q,p)\nonumber\\
 &=&\log \max_{i\in\{1,\ldots,d\}} \frac{p_i}{q_i}\max_{i\in\{1,\ldots,d\}} \frac{q_i}{p_i}\nonumber\\
&=& \log \frac{ \max_{i\in\{1,\ldots,d\}}\frac{p_i}{q_i}}{\min_{i\in\{1,\ldots,d\}} \frac{p_i}{q_i}}.\label{eq:rhohd}
\end{eqnarray}

\begin{property}
We can compute efficiently the Hilbert simplex distance in $\Delta_d$ in optimal $O(d)$ time.
\end{property}

The Hilbert simplex geometry yields a complete metric space which is a geodesic space, although the  geodesics are not unique in the Hilbert simplex geometry as depicted in Figure~\ref{fig:notgeodesic}.

\begin{figure}[tb]
\centering
\includegraphics[width=.5\columnwidth]{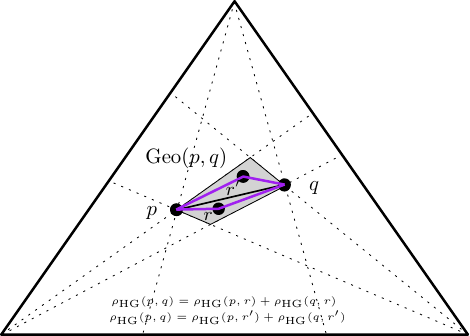}

\caption{Non-uniqueness of geodesics in the Hilbert simplex geometry:
The quadrilateral region $\mathrm{Geo}(p,q)$ denotes the set of points $r$ satisfying the triangle equality with respect to $p$ and $q$:
 $\rhohg(p,q)=\rhohg(p,r)+\rhohg(q,r)$. The purple paths connecting $p$ and $q$ are examples of geodesics. \label{fig:notgeodesic}}
\end{figure}

Figure~\ref{fig:shape} displays the shapes of balls with respect to the oriented Funk distances and the symmetrized Hilbert distance. Balls in the Hilbert simplex geometry
    have Euclidean polytope shapes of constant combinatorial complexity (e.g., hexagons~\cite{HG-SoCG-2017} in 2D which show that balls are convex but not strictly convex).
Since at infinitesimal scale, balls have polygonal shapes, it shows that the Hilbert simplex geometry is not Riemannian.
However, Hilbert geometries can be studied from the Finslerian point of view~\cite{troyanov2013funk} which generalizes Riemannian geometries.

\begin{figure}[tbh]
\centering
\begin{subfigure}{.98\columnwidth}
\includegraphics[width=\columnwidth]{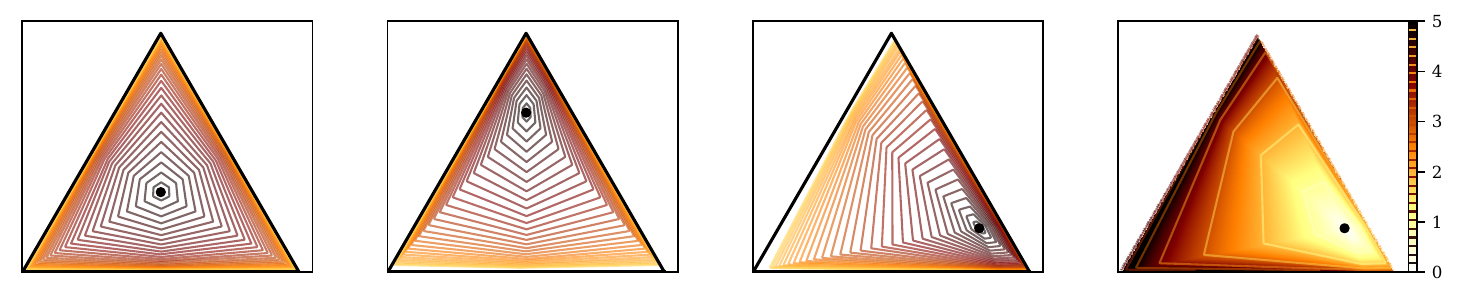}
\caption{$\rhohg(p,c)$}
\end{subfigure}
\begin{subfigure}{.98\columnwidth}
\includegraphics[width=\columnwidth]{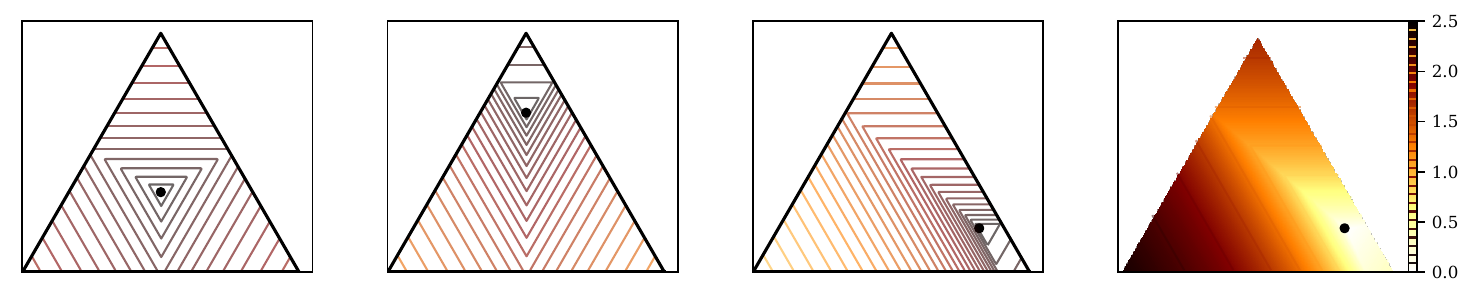}
\caption{$\rhofd(p,c)$}
\end{subfigure}
\begin{subfigure}{.98\columnwidth}
\includegraphics[width=\columnwidth]{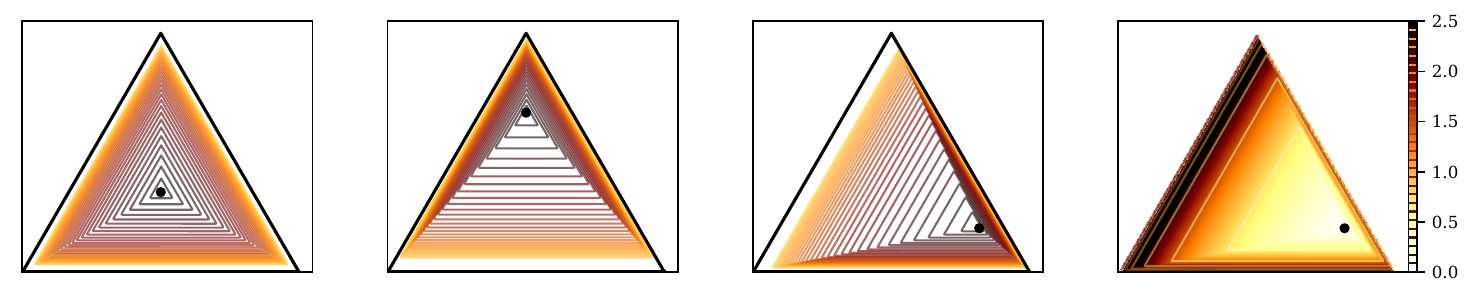}
\caption{$\rhofd(c,p)$}
\end{subfigure}

\caption{
    Balls centered at $c\in\Delta_2$ with constant radius increment step. The last column also shows the distance color maps (dark color means long distance).}%
\label{fig:shape}%
\end{figure}

\subsection{Monotone distance}\label{sec:monotone}

Let $\calX=\{X_1,\ldots, X_m\}$ be a partition of $\{1,\ldots, d\}$ into $m\leq d$ pairwise disjoint subsets $X_i$'s.
For $p\in\Delta_d$, let $p_{|\calX}\in\Delta_m$ denote the reduced dimension point with $p_{|\calX}[i]=\sum_{j\in X_i} p[i]$.
A distance $D(p,q)$ is said \emph{monotone}~\cite{amari2016information} iff
$$
D\left(p_{|\calX},q_{|\calX}\right)
\leq
D(p,q), \quad\forall\calX, \forall p,q\in\Delta_d.
$$

A distance is said \emph{separable} iff it can be expressed as a sum of scalar distances.
For example, the Euclidean distance is not separable but the squared Euclidean distance is separable.
The only separable monotone distances are $f$-divergences~\cite{amari2016information} when $d>2$.
The special curious case $d=2$ is dealt in~\cite{jiao2014information}.
We can interpret points in the simplex $\Delta_d$ as categorical distributions on a sample space of $d$ outcomes. Hence, the Hilbert statistical distance can also be said information monotone~\cite{amari2016information}.

We shall prove that the Funk oriented distance and the Hilbert distance are non-separable monotone distances.

\begin{lemma}\label{thm:funkmono}
    Let $p,q\in\Delta_d$. Let $\tilde{p}=(p_1+p_2,p_3,\cdots,p_d)$ and
    $\tilde{q}=(q_1+q_2,q_3,\cdots,q_d)$ denote their coarse-grained points on
    $\Delta_{d-1}$. We have $0\le \rhofd(\tilde{p},\tilde{q}) \le \rhofd(p,q)$.
\end{lemma}
\begin{proof}
    Denote $\iota=\max\{p_1/q_1, p_2/q_2\}$. As $q_1,q_2>0$, we have
    $p_1 \le \iota q_1$ and $p_2 \le \iota q_2$. Therefore
    \begin{equation*}
        \frac{p_1+p_2}{q_1+q_2} \le \frac{\iota q_1+ \iota q_2}{q_1+q_2} = \iota.
    \end{equation*}
    It follows that
\begin{align*}
        &\max\left\{
            \frac{p_1+p_2}{q_1+q_2}, \frac{p_3}{q_3}, \cdots, \frac{p_d}{q_d}
        \right\}\\
        &\hspace{5em}\le\;
        \max\left\{
            \frac{p_1}{q_1}, \frac{p_2}{q_2}, \frac{p_3}{q_3}, \cdots, \frac{p_d}{q_d}
        \right\}.
\end{align*}
    Hence
    \begin{equation*}
        \log
        \max\left\{
            \frac{p_1+p_2}{q_1+q_2}, \frac{p_3}{q_3}, \cdots, \frac{p_d}{q_d}
        \right\}
        \le \log\max_i \frac{p_i}{q_i}.
    \end{equation*}
    By the definition of the Funk distance, we get $\rhofd(\tilde{p},\tilde{q}) \le \rhofd(p,q)$.
\end{proof}

\begin{theorem}\label{thm:FunkHilbertMonotone}
    The Funk distance $\rhofd$ and the Hilbert distance $\rhohg$ in $\Delta_d$ satisfy the information monotonicity.
\end{theorem}
The proof is straightforward from Lemma~\ref{thm:funkmono} by noting that any
coarse-grained point can be recursively defined by merging two bins.
Since the sum of two information monotone distances is monotone, we get the proof that Hilbert distance as the sum of the forward and reverse Funk (weak) metric is monotone.

In fact, we can also prove this result by using Birkhoff's contraction mapping theorem~\cite{birkhoff1957extensions}.
We can represent the coarse-graining mapping $p\mapsto p_{|\calX}$ by a linear application with a $m\times d$ matrix $A$ with columns  summing  up to one (i.e., positive column-stochastic matrix):
$$
p_{|\calX} = A p.
$$
Then we have~\cite{birkhoff1957extensions}:
$$
\rhohg(Ap,Aq)\leq \tanh\left(\frac{1}{4}\Delta(A)\right)\, \rhohg(p,q),
$$
where $\Delta(A)$ is called the projective diameter of the positive mapping $A$:
$\Delta(A) := \sup\{\rhohg(Ap,Aq)\ :\ p,q\in\bbR_{++}^d\}$.

Since $0\leq \tanh(x)\leq 1$ for $x\geq 0$, we get the property that  Hilbert distance on the probability simplex is a monotone non-separable distance:
$\rhohg(p_{|\calX},q_{|\calX})\leq \rhohg(p,q)$.
Note that Birkhoff's contraction theorem is also used to prove the convergence of Sinkhorn's algorithm~\cite{peyre2019computational} used in entropy-regularized optimal transport.

Hilbert distance of Eq.~\ref{eq:rhohd} can be extended to the positive orthant cone $\bbR_{++}^d$ which can be foliated by homothetic simplices $\lambda\Delta_d$: $\bbR_{++}^d=\cup_{\lambda>0} \lambda\Delta_d$:
\begin{equation}
\rhohg(\tilde p,\tilde q)=\log \frac{ \max_{i\in\{1,\ldots,d\}}\frac{\tilde{p}_i}{\tilde{q}_i}}{\min_{i\in\{1,\ldots,d\}} 
\frac{\tilde{p}_i}{\tilde{q}_i}}, \quad\tilde{p},\tilde{q}\in\bbR_{++}^d.
\end{equation}
This extended Hilbert distance is projective because $\rhohg(\alpha\tilde p,\beta\tilde q)=\rhohg(\tilde p,\tilde q)$.
Thus the Hilbert distance is a projective distance between rays $\tilde{p}$ and $\tilde{q}$ and a metric distance on  $\lambda\Delta_d$
 for any prescribed value of $\lambda>0$.
This Hilbert projective distance is also called the Birkhoff distance~\cite{BH-2014} in the literature.
Furthermore, any projective distance with strict contraction by linear transformation is provably a scalar function of the Hilbert's projective metric, and Hilbert projective distance is shown to have the lowest  possible contraction bound among proper cone projective distances~\cite{UniversalHilbertProjectiveMetric-1982}.

The Aitchison distance~\cite{pawlowsky2011compositional} is also a non-separable distance in the standard simplex defined as follows:
\begin{equation}\label{eq:AD}
\rho_{\mathrm{Aitchison}}(p,q)
\eqdef
\sqrt{\sum_{i=1}^d \left( \log \frac{p_i}{G(p)}-\log \frac{q_i}{G(q)} \right)^2},
\end{equation}
where $G(p)$ denotes the geometric mean of the coordinates of $p\in\Delta_d$:
$$
G(p)=\left(\prod_{i=1}^d p_i\right)^{\frac{1}{d}}=\exp\left(\frac{1}{d}\sum_{i=1}^d \log p_i\right).
$$
The Aitchison distance is non-separable and satisfies the monotonicity property~\cite{CODA-2021}.

\subsection{Isometry to a normed vector space}\label{sec:isometryVS}

Hilbert geometry is never a Hilbert space (i.e., complete metric space induced by the inner product of a vector space) because the convex domain $\Omega$ is bounded.
It can be shown that the only domains $\Omega$ yielding an isometry of the Hilbert geometry to a normed vector space are simplices~\cite{QuasiIsometryNormedVectorSpace-HilbertGeometry-2008,foertsch2005hilbert}.

We recall the isometry~\cite{HilbertHarpe-1991} of the standard simplex to a normed vector space $(V_d,\|\cdot\|_\NH)$.
Let $V_d=\{v\in\bbR^{d} \st \sum_{i=1}^d v_i=1\}$ denote the $(d-1)$-dimensional vector space sitting in $\bbR^{d}$.
Map a point $p=(p_1,\ldots,p_{d})\in\Delta_d$ to a point $v(p)=(v_1,\ldots, v_{d})\in V_d$ as follows:
\begin{eqnarray*}
v_i
&=& \frac{1}{d} \left((d-1)\log p_i -\sum_{j\neq{i}}\log p_j \right),\\
&=& \log p_i - \frac{1}{d}\sum_{j=1}^d \log\ p_j=\log\frac{p_i}{G(p)}.
\end{eqnarray*}

We define the corresponding norm $\|\cdot\|_\NH$ in $V_d$ by considering the shape of its unit ball
$$
B_V=\{v\in V_d \st |v_i-v_j|\leq 1, \forall i\not =j\}.
$$
The unit ball $B_V$ is a symmetric convex set containing the origin in its interior, and thus yields a {\em polytope norm}
 $\|\cdot\|_\NH$ (Hilbert norm) with $2\binom{d}{2}=d(d-1)$ facets.
 Norms $\ell_1$ and $\ell_\infty$ yield hypercube balls with $2d$ facets and $2^d$ vertices.
Reciprocally, let us notice that a norm induces a unit ball centered at the origin that is convex and symmetric around the origin.
The distance in the normed vector space between $v\in V_d$ and $v'\in V_d$ is defined by:
\begin{equation*}
\rho_V(v,v')= \|v-v'\|_\NH = \inf \left\{ \tau \st v'\in \tau(B_V\oplus \{v\}) \right\},
\end{equation*}
where $A\oplus B=\{a+b \st a\in A,b\in B\}$ is the Minkowski sum of sets.
Figure~\ref{fig:boundnorm} illustrates the balls centered at the origin with respect to the polytope norm $\|\cdot\|_\NH$.

\begin{figure}
\includegraphics[width=\columnwidth]{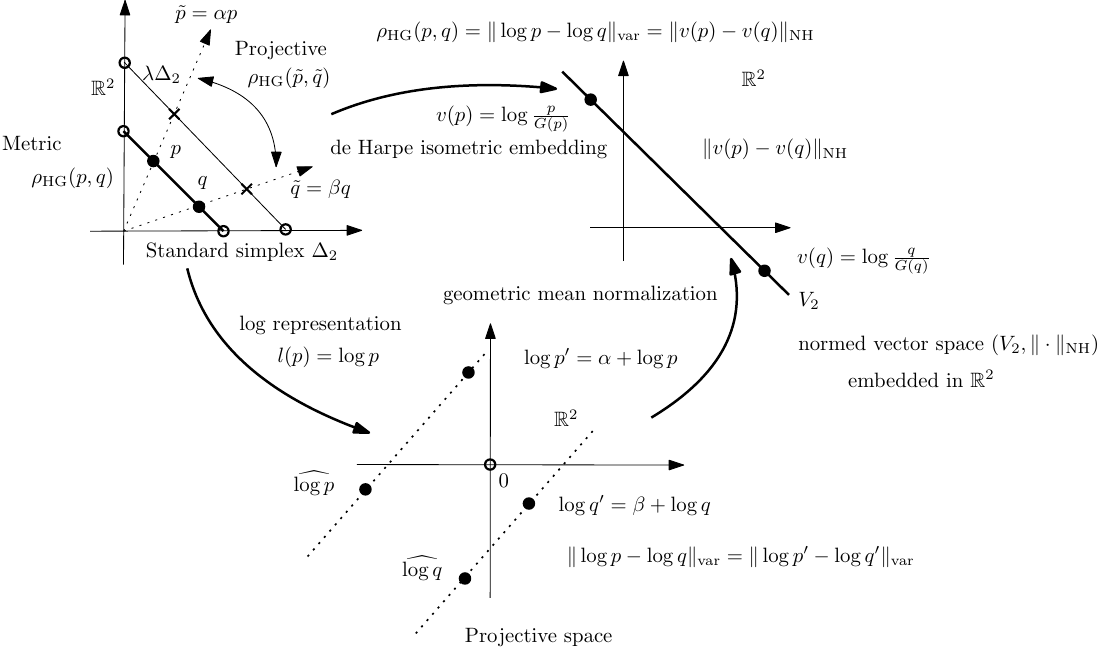}
\caption{Different representations of the simplex and positive orthant cone.}\label{fig:simplextransfer}
\end{figure}

\def\bbL{\mathbb{L}}

Let $l(p)=(\log p_1,\ldots, \log p_d) \in\bbR^d$ be the logarithmic mapping
and $\bbL_d=\{l(p) \ :\ p\in\Delta_d \}$.
We have
$$
\rho_\HG(p,q) = \|l(p)-l(q)\|_\var = \|l(\tilde{p})-l(\tilde{q})\|_\var,
$$
for any $\alpha,\beta\in\bbR$ with $\tilde{p}=\alpha p$, $\tilde{q}=\beta q$, and where 
$$
\|x\|_\var:=\max_i x_i-\min_i x_i=\|x\|_{+\infty}-\|x\|_{-\infty}
$$ 
is the variation semi-norm.
$\|\cdot\|_\var$ is only a {\em semi-norm} because we have
$\|(\lambda,...,\lambda)\|_\var =0$ for any $\lambda\in\bbR$.

Thus to convert from $\bbL$ to $\Delta_d$, we need to find the representative
element of the equivalence class $\hat{l}$ of $l$:
normalize $l\in\bbL$ by 
$$
p(l)=\hat{l}=\exp(l)/\sum_{i=1}^d e^{l_i}.
$$
Then we convert $\hat{l}$ to $v(\hat{l})$ by choosing translation $-\log
G(\hat{l})$, where $G$ is the geometric mean:
\begin{equation}
v(l)
=
\left(
l_1-\log G\left(\frac{e^l}{\sum_{i=1}^d e^{l_i}}\right),
\ldots,
l_d-\log G\left(\frac{e^l}{\sum_{i=1}^d e^{l_i}}\right)
\right).
\end{equation}

Thus we have
\begin{eqnarray*}
\lefteqn{\rho_\HG(p,q)=\rho_\HG(\tilde{p},\tilde{q})
=\|\log\tilde{p}-\log\tilde{q}\|_\var}\nonumber\\
&= \|v(p)-v(q)\|_{\NH}
= \|l(\hat{p})-l(\hat{q})\|_{\NH}.
\end{eqnarray*}

Therefore
$$
\|l-l'\|_\var = \|v(\hat{l})-v(\hat{l'})\|_{\NH}.
$$
Figure~\ref{fig:simplextransfer} illustrates different transformations of the
simplex space included in the positive orthant cone.

\begin{figure}
\centering
\includegraphics[width=0.5\columnwidth]{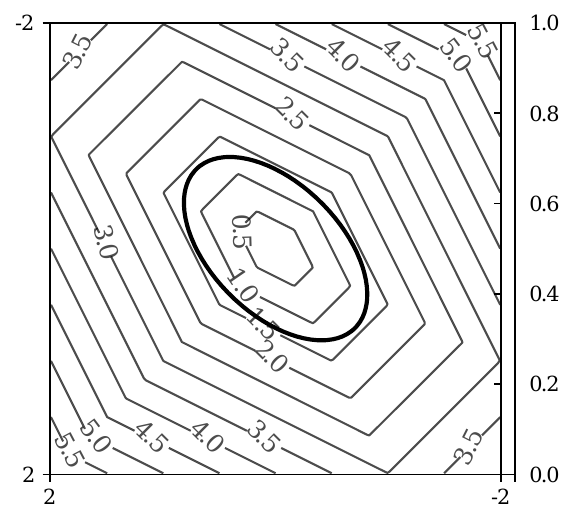}
\caption{Polytope balls $B_V$ and the Euclidean unit ball $B_E$ shown on the 2D slanted plane
$V^3=\{v\in\mathbb{R}^3\ :\ \sum_{i=1}^3 v^i=0\}$ of $\bbR^3$.
}\label{fig:boundnorm}
\end{figure}

The reverse map from the normed space $V_d$ to the standard simplex $\Delta_d$ is given by
the softmax function:
\begin{equation*}
p_i = \frac{\exp({v_i})}{\sum_j \exp(v_j)}.
\end{equation*}

Thus we have $(\Delta_d,\rho_\HG)\cong (V_d,\|\cdot\|_\NH)$.
In 1D, $(V^2,\|\cdot\|_\NH)$ is isometric to the Euclidean line.

Now, let us notice that coordinate $v_i$ can be rewritten as
$$
v_i=\log \frac{p_i}{G(p)},
$$
where $G(p)$ is the coordinate geometric means.
Recalling that the Hilbert simplex distance is a projective distance on the positive orthant cone domain:
\begin{eqnarray*}
\rho_{\HG}(p,q) &=& \log \frac{\max _{i\in\{1,\ldots, d\}} \frac{p_{i}}{q_{i}}}{\min _{j\in\{1,\ldots, d\}} \frac{p_{j}}{q_{j}}},\\
&=&
\rho_{\HG}(\lambda p,\lambda' q), \quad\forall \lambda>0,\lambda'>0.
\end{eqnarray*}

Thus we have:
\begin{eqnarray*}
\rho_{\HG}(p,q) &=& \|\log p-\log q\|_{\var},\\
&=& \|\log (\lambda p)-\log (\lambda' q)\|_{\var}, \quad\forall \lambda>0,\lambda'>0.
\end{eqnarray*}

Choose $\lambda=\frac{1}{G(p)}$ and $\lambda'=\frac{1}{G(q)}$ to get
$$
\rho_{\HG}(p,q) = \left\|\log \frac{p}{G(p)}-\log \frac{q}{G(q)}\right\|_{\var}.
$$

This highlights a nice connection with the Aitchison distance of Eq.~\ref{eq:AD}:
\begin{align}
\rhohg(p,q)
&= \left\Vert \log\frac{p}{G(p)}-\log\frac{q}{G(q)} \right\Vert_{\var},\\
\rho_{\mathrm{Aitchison}}(p,q)
&= \left\Vert \log\frac{p}{G(p)}-\log\frac{q}{G(q)} \right\Vert_{2}.
\end{align}

Thus both the Aitchison distance and the Hilbert simplex distance are normed distances on the representation
$$
p\mapsto \log\frac{p}{G(p)}=\left(\log\frac{p_1}{G(p)},\ldots,\frac{p_d}{G(p)}\right).
$$
Notice that since the geometric mean is homogeneous, i.e., we have 
$\log \frac{p}{G(p)} =\log \frac{\lambda p}{G(\lambda p)}, \forall \lambda>0$.

Figure~\ref{fig:VoronoiAitchisonHilbertVarNorm} displays the Voronoi diagram of $n=16$ points in the probability simplex with respect to the Aitchison distance (Figure~\ref{fig:VoronoiAitchisonHilbertVarNorm}, left), and the Hilbert simplex distance (Figure~\ref{fig:VoronoiAitchisonHilbertVarNorm}, middle) and its equivalent variation norm space by logarithmic embedding (Figure~\ref{fig:VoronoiAitchisonHilbertVarNorm}, right).
See also~\cite{VoronoiHilbert-2023}.

The Hilbert simplex bisectors are piecewise linear in the normalized logarithmic representation since they are induced by the  polyhedral semi-norm   $\|.\|_\var$, and thus are structurally more complex than the Aitchison bisectors
which are linear/affine in the $\log x/G(x)$ representation:
Aitchinson Voronoi diagram can be derived from an ordinary Euclidean Voronoi diagram~\cite{boissonnat1998voronoi}.

\begin{figure}
\centering
\includegraphics[width=.95\columnwidth]{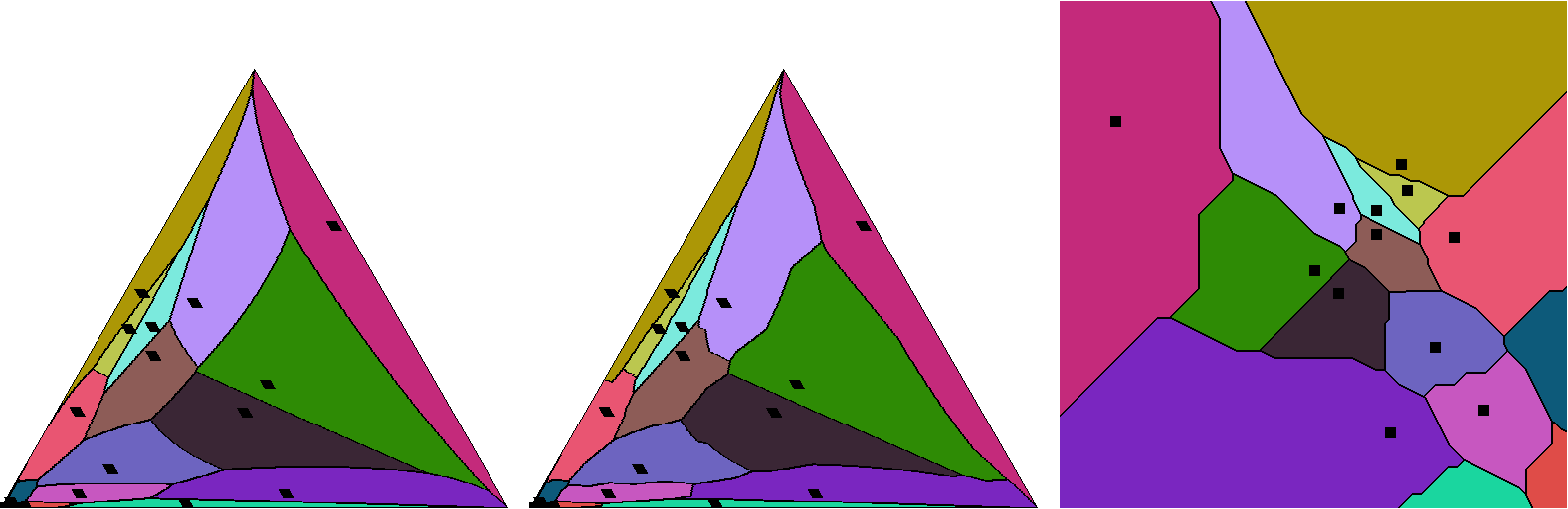}
\caption{Voronoi diagram in the probability simplex with respect to the Aitchison distance (left), Hilbert simplex distance (middle) and equivalent variation-norm induced distance on normalized logarithmic representations.}\label{fig:VoronoiAitchisonHilbertVarNorm}

\end{figure}

\subsection{Differentiable approximation}
\def\LSE{\mathrm{LSE}}
\def\LSET{{\mathrm{LSE}^T}}

\begin{figure*}[tbhp]
\centering
\includegraphics[width=.95\textwidth]{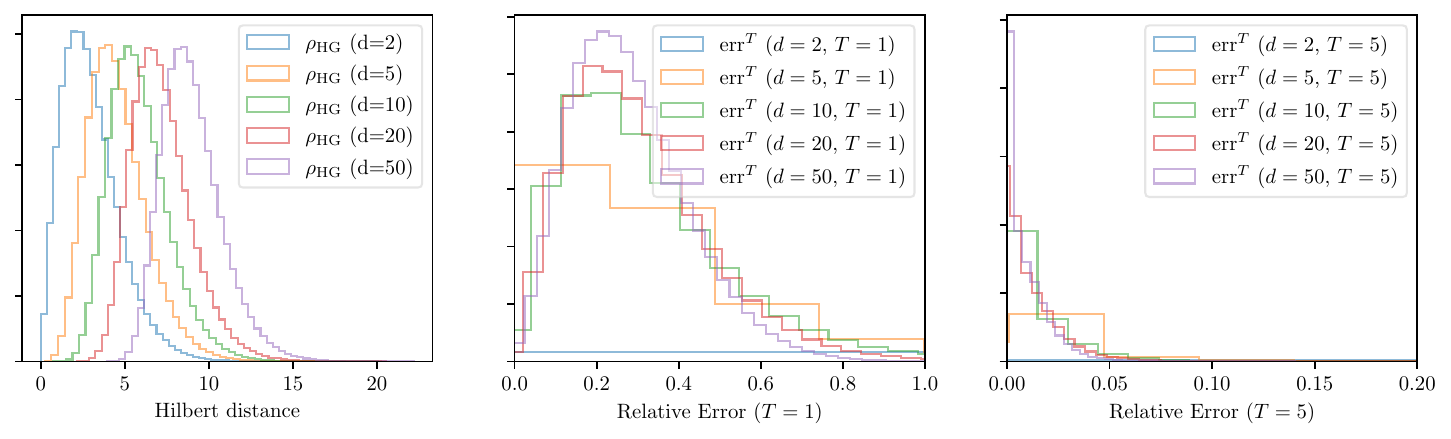}
\caption{Histograms of Hilbert distances and the relative errors
on $10^6$ pairs of uniform random points in $\Delta_d$.}
\label{fig:bound}
\end{figure*}

The Hilbert simplex distance is not differentiable because of the max operations.
However, since the logarithm function is strictly increasing, we can rewrite the Funk distance as
$$
\rhofd(p,q)=\log \max_i \frac{p_i}{q_i}=\max_i \log\frac{p_i}{q_i}.
$$
In machine learning, the log-sum-exp function
\begin{equation*}
\LSE(x_1,\ldots, x_d)
\;\eqdef\;
\log \left(\sum_{i=1}^d\exp(x_i)\right)
\end{equation*}
is commonly used to differentiably approximate the maximum operator.
In fact, one can use the general approximation formula
\begin{align*}
\LSET(x_1, \ldots, x_d)
&\;\eqdef\;
\frac{1}{T}
\LSE(Tx_1, \ldots, T x_d)\nonumber\\
&=
\frac{1}{T}
\log \left(\sum_{i=1}^d\exp(T x_i)\right),
\end{align*}
where $T>0$. For any $x\in\Re^d$, we have the approximation bounds
\begin{align}
\max_i x_i + \varepsilon_1(x, T)
& \;\le\;
\LSET(x_1,\ldots, x_d) \nonumber\\
& \;\le\;
\max_i x_i + \varepsilon_2(x, T),\label{eq:bounds}
\end{align}
where
\begin{align*}
\varepsilon_1(x,T)
&:=
\frac{1}{T}\log\left[
    1+(d-1) \exp(-T\| x \|_\var)
    \right],\nonumber\\
\varepsilon_2(x,T)
&:=
\frac{1}{T} \log\left[
    d-1 + \exp(-T\| x \|_\var)
    \right].
\end{align*}
Obviously, 
$$0<\varepsilon_1(x,T)\le\varepsilon_2(x,T)\le\frac{1}{T}\log{d}
$$
and both $\varepsilon_1(x,T)$ and $\varepsilon_2(x,T)$
tend to 0 as $T$ increases, making the approximation
$\LSET(x_1,\ldots, x_d)$ accurate.

Because $\forall{i}$, $x_i\ge \max_i x_i - \|x\|_\var$, we have
\begin{align*}
&\left(\sum_{i=1}^d\exp(T x_i)\right)
\ge
(d-1)\exp(T \max_i x_i - T \|x\|_\var)\nonumber\\
& \hspace{10em} + \exp(T \max_i x_i)\nonumber\\
&=
( (d-1) \exp( - T \|x\|_\var ) + 1 ) \exp(T \max_i x_i).
\end{align*}
Taking the logarithm on both sides gives
the first ``$\le$'' in Eq.~(\ref{eq:bounds}).
The proof of the second ``$\le$'' is similar.

Thus we have
\begin{align}\label{eq:fdbound}
\rhofd(p,q) + \varepsilon_1(r, T)
&\leq
\frac{1}{T}
\log\left(\sum_i \left(\frac{p_i}{q_i}\right)^T \right)
\nonumber\\
&\leq
\rhofd(p,q)
+ \varepsilon_2(r, T),
\end{align}
where $r_i=\log p_i - \log q_i$.

Hence, we may define a differentiable pseudo-distance by symmetrizing the $\LSET$ function:
\begin{equation}
 \tilde\rho_\LSET(p,q)= 
\frac{1}{T}
\log\left(\sum_i \left(\frac{p_i}{q_i}\right)^T \right)
\left(\sum_i \left(\frac{q_i}{p_i}\right)^T \right).\label{eq:lsetapprox}
\end{equation}

We can rewrite $\tilde\rho_\LSET(p,q)$ using the $\|\cdot\|_T$ norms as follows:
$$
\tilde\rho_\LSET(p,q)=\log \left\|\frac{p}{q}\right\|_T \, \left\|\frac{q}{p}\right\|_T,
$$
where $\frac{p}{q}=\left(\frac{p_1}{q_1},\ldots,\frac{p_d}{q_d}\right)$ and 
$\frac{q}{p}=\left(\frac{q_1}{p_1},\ldots,\frac{q_d}{p_d}\right)$.

Since $\lim_{T\rightarrow\infty} \|x\|_T=\max_i x_i$, we deduce that
$\lim_{T\rightarrow\infty}  \tilde\rho_\LSET(p,q)=\rho_\HG(p,q)$:

\begin{proof}
\begin{eqnarray*}
\lim_{T\rightarrow\infty}  \tilde\rho_\LSET(p,q) &=&  \log \left(\max_i \frac{p_i}{q_i}\right)\left(\max_i \frac{q_i}{p_i}\right),\\
&=& \log \max_i \frac{p_i}{q_i} + \log \max_i \frac{q_i}{p_i},\\
&=& \log \max_i \frac{p_i}{q_i} + \max_i \log \frac{q_i}{p_i},\\
&=& \log \max_i \frac{p_i}{q_i} + \max_i \left(-\log \frac{p_i}{q_i}\right),\\
&=& \log \max_i \frac{p_i}{q_i} - \min_i \log \frac{p_i}{q_i},\\
&=&  \log \max_i \frac{p_i}{q_i} - \log \min_i \frac{p_i}{q_i},\\
&=& \log \frac{\max_i \frac{p_i}{q_i}}{\min_i \frac{p_i}{q_i}},\\
&=:& \rho_\HG(p,q).
\end{eqnarray*}
\end{proof}

\begin{figure}[t!]
\centering
\begin{tabular}{ccc}
\includegraphics[width=0.9\columnwidth]{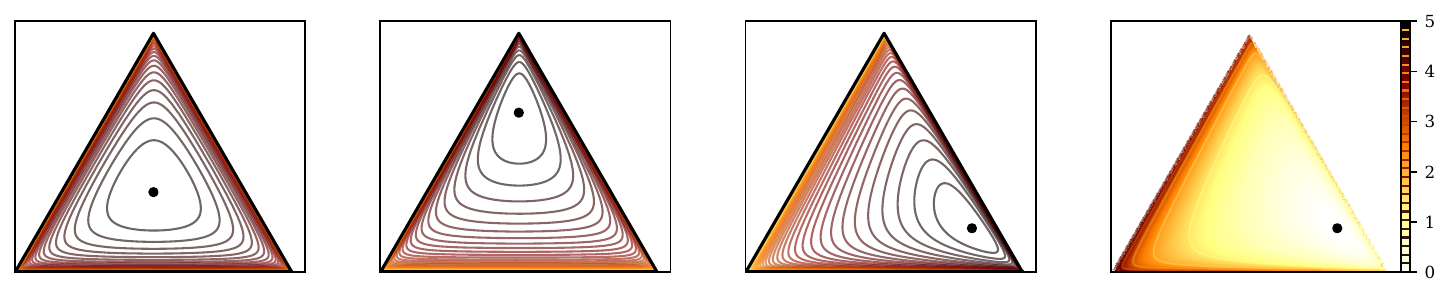} \\
$T=1$\\
\includegraphics[width=0.9\columnwidth]{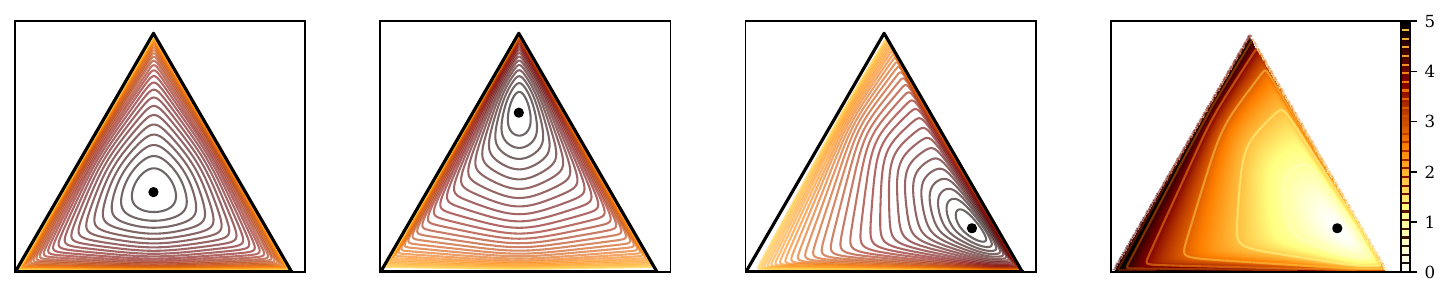} \\
$T=5$\\
\includegraphics[width=.9\columnwidth]{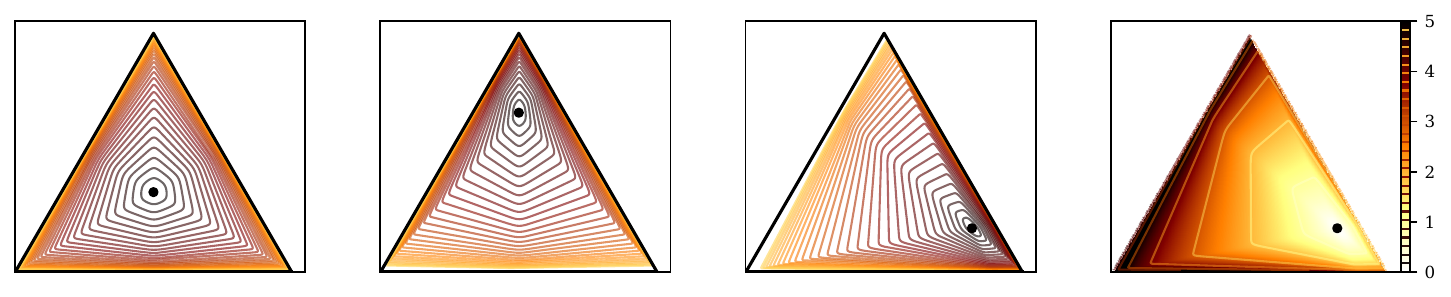} \\
$T=20$\\
\includegraphics[width=.9\columnwidth]{hilbert_balls}\\
$T\rightarrow\infty$
\end{tabular}
\caption{Balls of various radii with respect to $\tilde\rho_\LSET$ for $T\in\{1,5,20\}$ (compare with Figure~\ref{fig:shape}).
The last row shows the ball shapes with respect to Hilbert simplex distance (i.e., $T\rightarrow\infty$).\label{fig:ballT}}
\end{figure}

We can further write
\begin{equation}
\tilde\rho_\LSET(p,q)=\log \frac{\left\|\frac{p}{q}\right\|_T}{\left\|\frac{p}{q}\right\|_{-T}}.
\end{equation}

Since $\|1\|_T=d^{\frac{1}{T}}$, we have $\tilde\rho_\LSET(p,q)=\frac{2}{T}\log d$.
Notice that we may also write $\|x\|_T=d^{\frac{1}{T}}\, m_T(x_1,\ldots, x_d)$ where $m_T$ is the power mean:
$m_T(x)=m_T(x_1,\ldots, x_d)=\left(\frac{1}{d} \sum_{i=1}^d x_i^T \right)^{\frac{1}{T}}$.
Thus we can also write
\begin{equation}
\tilde\rho_\LSET(p,q)=\frac{2}{T} \log d+\log \frac{m_T\left(\frac{p}{q}\right)}{m_{-T}\left(\frac{p}{q}\right)}.
\end{equation}

Thus we have $\tilde\rho_\LSET(p,q)\geq 0$ and $\tilde\rho_\LSE(p,p)=\frac{2}{T}\log d$.
Similar to Eq.~(\ref{eq:fdbound}), we have
\begin{align*}
\rhohg(p,q) + 2\epsilon_1(r,T)
&\le
\tilde\rho_\LSET(p,q)\\
&\le
\rhohg(p,q)
+2 \epsilon_2(r,T).
\end{align*}

Since $m_{T'}(p/q)\geq m_{T}(p/q)$ whenever $T'\geq T$, we have
$$
\log \frac{m_{T'}(p/q)}{m_{-T'}(p/q)} \geq \log \frac{m_{T}(p/q)}{m_{-T}(p/q)}. 
$$
Thus we have
$\rho_\LSET(p,q)=\tilde\rho_\LSET(p,q)-\frac{2}{T} \log d$ which increases monotonically with $T$.

Figure~\ref{fig:ballT} illustrates the shape of balls in the standard simplex with respect to $\tilde\rho_\LSET$:
Observe that as $T$ decreases the ball shapes become roundish and as $T$ increases the ball shapes look like the ball shapes in Hilbert simplex geometry visualized in Figure~\ref{fig:shape}. 

The maximum deviation from the approximation
$\tilde\rho_\LSET(p,q)$
to the true Hilbert distance
$\rhohg(p,q)$ is bounded by $\frac{2}{T}\log{d}$.

Figure~\ref{fig:bound} shows the histogram of Hilbert distances on
uniform random points drawn from $\Delta_d$, and the relative error
\begin{equation*}
\mathrm{err}^T(p,q)
\eqdef
\frac{\tilde\rho_\LSET(p,q) - \rhohg(p,q)}{\rhohg(p,q)}
\end{equation*}
from the differentiable approximation $\tilde\rho_\LSET(p,q)$
to the true Hilbert distance $\rhohg(p,q)$. We can verify empirically
that $\tilde\rho_\LSET(p,q)$ is always larger than $\rhohg(p,q)$.
We observe the approximation becomes more accurate when $T$ increases from 1 to 5.
The approximation error tends to be large at a small dimensionality $d$.

\section{Comparing different geometries}\label{sec:exp}

\begin{figure}[ht!]
\centering
\begin{subfigure}{.3\textwidth}
\centering
\includegraphics[width=\textwidth]{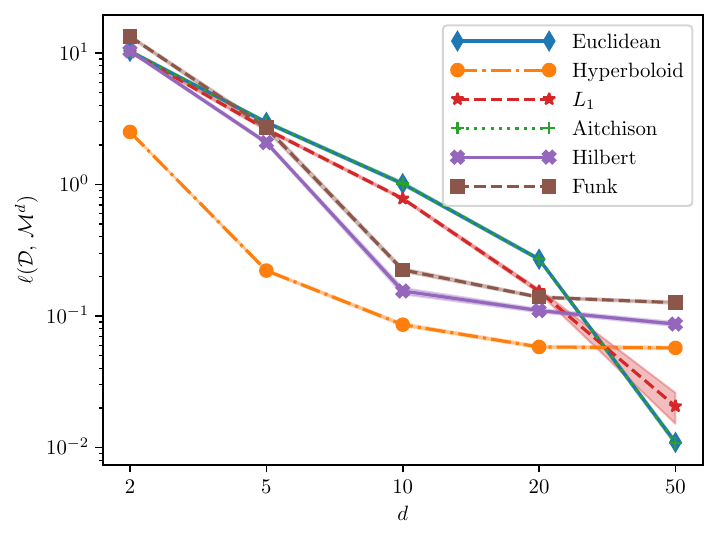}
\caption{100 random points in $\Re^{100}$}
\end{subfigure}
\begin{subfigure}{.3\textwidth}
\centering
\includegraphics[width=\textwidth]{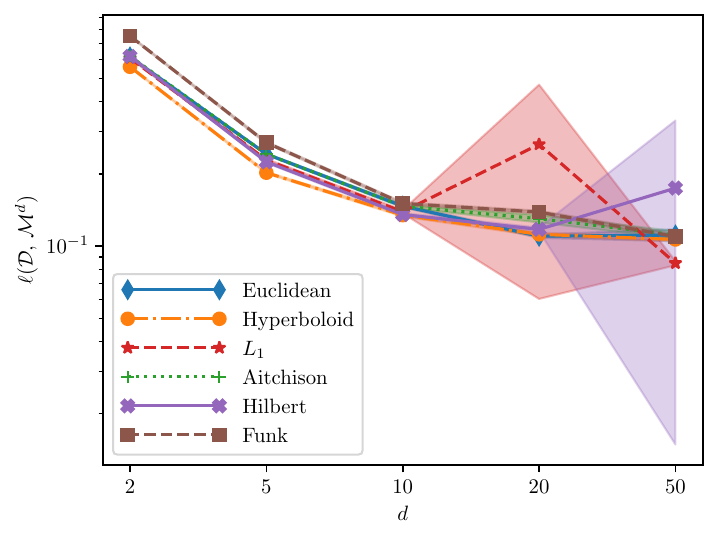}
\caption{Erd\H{o}s--R\'enyi graphs $G(n,p)$ ($n=200$, $p=0.2$)}
\end{subfigure}
\begin{subfigure}{.3\textwidth}
\centering
\includegraphics[width=\textwidth]{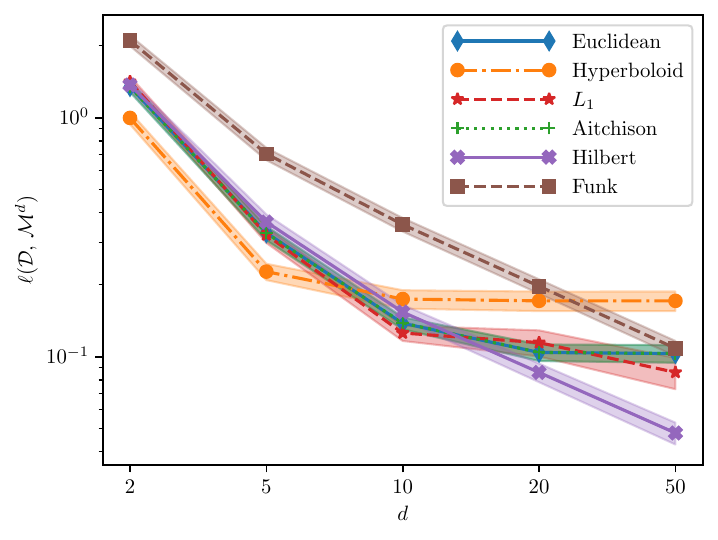}
\caption{Barab\'asi--Albert graphs $G(n,m)$ ($n=200$, $m=2$)}
\end{subfigure}

\caption{Embedding loss $\ell(\calD,\calM^d)$ against the
    number of dimensions $d$ on three random datasets.\label{fig:lossd}}
\end{figure}

\begin{figure}[tbh]
\centering

\begin{subfigure}{.3\textwidth}
\centering
\includegraphics[width=\textwidth]{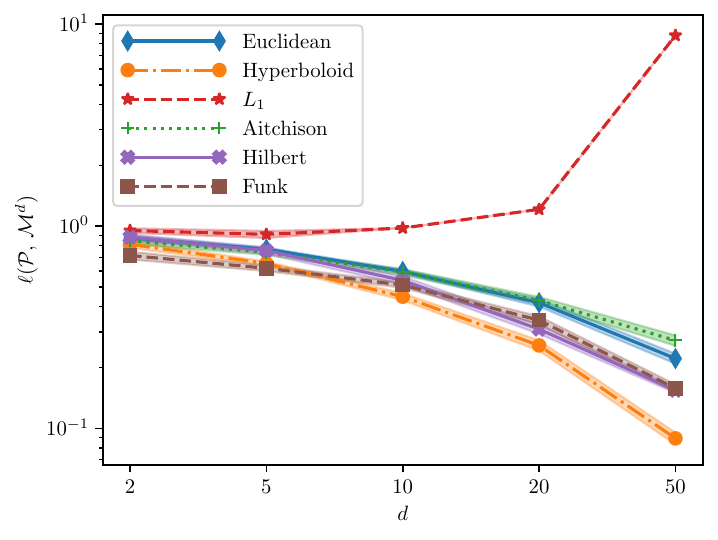}
\caption{100 random points in $\Re^{100}$}
\end{subfigure}
\begin{subfigure}{.3\textwidth}
\centering
\includegraphics[width=\textwidth]{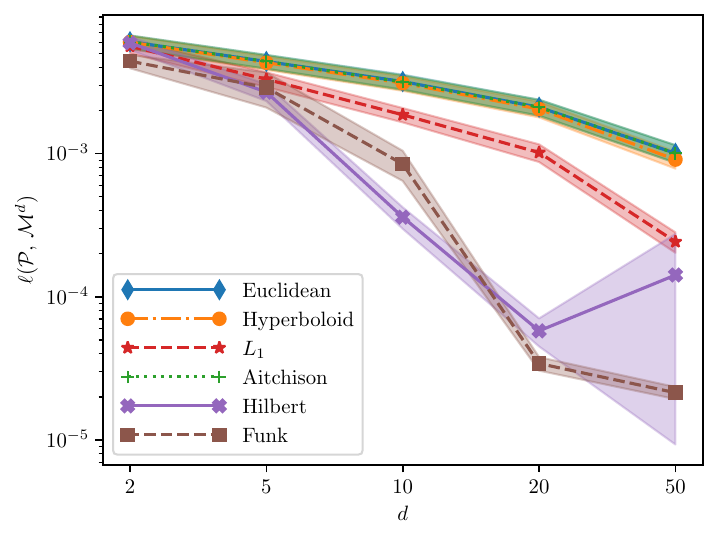}
\caption{Erd\H{o}s--R\'enyi graphs $G(n,p)$ ($n=200$, $p=0.2$)}
\end{subfigure}
\begin{subfigure}{.3\textwidth}
\centering
\includegraphics[width=\textwidth]{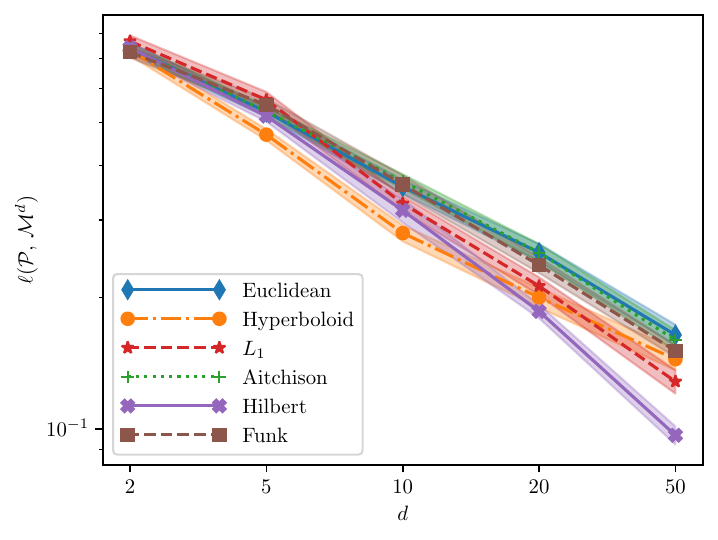}
\caption{Barab\'asi--Albert graphs $G(n,m)$ ($n=200$, $m=2$)}
\end{subfigure}

\caption{Embedding loss $\ell(\calP,\calM^d)$ against
    the number of dimensions $d$.}\label{fig:lossp}
\end{figure}

We empirically compare the representation power of different geometries for embedding the input data
as a set of points on the manifold. Our objective is not to build a
full-fledged embedding method, but to have simple
well-defined measurements to \emph{compare different geometries}.
Notice that if $(M_1,\rho_1)$ is isometric to $(M_2,\rho_2)$ then these geometries have the same representation power.
We prefer to choose model geometries with unconstrained domains for optimization.
Thus the Poincar\'e hyperbolic ball  and the Aitchison simplex embeddings are considered via the equivalent Minkowski hyperboloid and Euclidean models, respectively.

We consider embedding two different types of data onto a manifold $\calM^d$ of dimensionality $d$, which we also denote as $\calM$. The first is given by a distance matrix $\calD_{n\times{n}}$. The representation loss associated with $\calM^d$ is
\begin{equation*}
\ell(\calD, \calM^d)
\;\eqdef\;
\inf_{\bm{Y}\in(\calM^d)^n}
\frac{1}{n^2}\sum_{i=1}^n \sum_{j=1}^n
\left( \calD_{ij}
    - \rho_{\calM}(\bm{y}_i,\bm{y}_j) \right)^2,
\end{equation*}
where $\bm{Y}=\{\bm{y}_i\}_{i=1}^n$
is a set of $n$ free points on $\calM^d$,
and $\rho_{\calM}$ is the distance on $\calM$.
The infimum means the error associated with
the best representation of the given distance matrix.
A smaller value of $\ell(\calD, \calM^d)$
means $\calM^d$ can better represent the distance matrix $\calD$.

\section{Experimental Configurations}\label{app:expconfig}

We rewrite the loss as
\begin{align}
\ell(\calD, \calM^d) &\;\eqdef\;
\inf_{\bm{Y}\in(\calM^d)^n} L(\calD, \calM^d, \bm{Y})\nonumber\\
L(\calD, \calM^d, \bm{Y})
&=
\frac{1}{n^2}\sum_{i=1}^n \sum_{j=1}^n
\left( \calD_{ij}
    - \rho_{\calM}(\bm{y}_i,\bm{y}_j) \right)^2.\\
\ell(\calP, \mathcal{M}^d)
&\;\eqdef\;
\inf_{\bm{Y}\in(\calM^d)^n}
L(\calP, \calM^d, \bm{Y})\nonumber\\
L(\calP, \calM^d, \bm{Y})
&=
\frac{1}{n}\sum_{i=1}^n \sum_{j=1}^n
\calP_{ij} \log\frac{\calP_{ij}}{q_{ij}(\bm{Y})}.
\end{align}
The functions to be minimized,
$L(\calD, \calM^d, \bm{Y})$ and $L(\calP, \calM^d, \bm{Y})$,
are both expressed in the form of a sample average.
Therefore they can be optimized
based on stochastic gradient descent (SGD).

In the experiments, we use {\tt PyTorch} to minimize
$L(\calD, \calM^d, \bm{Y})$ and $L(\calP, \calM^d, \bm{Y})$
with respect to the coordinate matrix $\bm{Y}$.
The initial $\bm{Y}_0$ is based on a multivariate Gaussian
distribution so that the trace of the covariance matrix equals 1.

The optimizer is {\tt Adam}~\cite{Adam-2015} in its default settings except the learning rate.
The learning rate is based on a log-uniform distribution
(logarithm of the learning rate is uniform)
in the range $[10^{-3}, 1]$.
The mini-batch size is simply set to 16. We observed that reducing the
mini-batch size can generally achieve a smaller loss for all the methods.
For each configuration of (dataset, manifold $\calM^d$, dimensionality $d$),
the optimal learning rate is selected based on a Tree Parzen estimator with 20 trials.
The maximum number of epochs is $3000$. We use early stopping
to terminate the optimization process if convergence is detected.

Each dataset is generated independently for 10 times, based on different random
seeds. The loss for each of these generated dataset is computed independently.
The average and standard deviation are reported.

For all simplex embeddings (Hilbert, Funk, Aitchison),
we represent the embedding in the log-coordinates
$l(p)=(\log{}p_1,\dots,\log{p}_d)$.
Because
$\rho_\HG(p,q) = \|l(p)-l(q)\|_\var = \|l(\tilde{p})-l(\tilde{q})\|_\var$,
we can directly optimize the Hilbert simplex embedding on the coordinates
$l(\tilde{p})$ which are free points in $\Re^d$.
For Funk and Aitchison embeddings, we need to ensure 
that the embedding to be optimized can be mapped back into the simplex domain.


We set $\calD$ to be
\ding{192}
the distance matrix between $n$ random points in $\Re^{n}$ ($n=100$);
or \ding{193} the pairwise shortest path
between any two nodes
on an Erd\H{o}s--R\'enyi graph $G(n,p)$ ($n=200$, $p=0.2$);
or \ding{194} the node distance on a Barab\'asi--Albert
graph $G(n,m)$ ($n=200$, $m=2$).

On the other hand,
we can evaluate the geometries based on a
given probability matrix
$\calP_{n\times{n}}$, meaning some non-negative pair-wise similarities.
$\calP$ is row-normalized so that each row sums to 1.
We consider the loss
\begin{align*}
\ell(\calP, \mathcal{M}^d)
&\;\eqdef\;
\inf_{\bm{Y}\in(\calM^d)^n}
\frac{1}{n}\sum_{i=1}^n \sum_{j:j\neq{i}}
\calP_{ij} \log\frac{\calP_{ij}}{q_{ij}(\bm{Y})},\\
q_{ij}(\bm{Y})
&\;\eqdef\;
\frac{\exp(-\rho_{\calM}^2(\bm{y}_i,\bm{y}_j)} {\sum_{j:j\neq{i}}\exp(-\rho_{\calM}^2(\bm{y}_i,\bm{y}_j))},
\end{align*}
where $\ell(\calP,\calM^d)$ is
the empirical average of the KL divergence
between the probability distributions
$\calP_{i\cdot}$ and $q_{i\cdot}$.
Notice that $\ell$ is abused to denote
both the loss associated with a distance matrix
$\calD$ and a probability matrix $\calP$.
Using the same datasets as in embedding $\mathcal{D}$,
we set $\calP$ to be
\ding{192} pairwise similarities of $n$ random points in $\Re^n$
measured by the heat kernel after normalization;
\ding{193} the random walk similarity
starting from node $i$ and ending at any other
node $j$ after 5 steps on an Erd\H{o}s--R\'enyi graph,
or \ding{194} a Barab\'asi--Albert graph.

The embedding losses $\ell(\calD,\calM^d)$
and $\ell(\calP,\calM^d)$ are approximated based
on the Adam optimizer~\cite{adam}.
We minimize the function whose infimum is to
be taken with respect to $\bm{Y}$, starting from
some randomly initialized points $\bm{Y}_0$,
until converging to a local optimum.
The loss $\ell(\calD, \calM^d)$ is similar to
the stress function in multi-dimensional scaling~\cite{mds},
while $\ell(\calP, \calM^d)$ is similar to
the losses in manifold learning~\cite{sne}
or graph embedding~\cite{deepwalk}.
Our losses do not depend on many practical techniques
such as negative sampling, and are helpful to
measure the fitness of the manifold $\calM^d$ to the input
$\calD$ or $\calP$ regardless of these practical aspects.
The detailed experimental protocols and
more extensive results are in~\cite{nielsen2022non}.

Figure~\ref{fig:lossd} shows $\ell(\calD,\calM^d)$ (in log scale)
against $d$ for six different choices of $\calM$:
\ding{192} $\Re^d$ with Euclidean norm;
\ding{193} $\Re^d$ with $L_1$ norm;
\ding{194} Poincar\'e/Minkowski hyperboloid;
\ding{195} $\Delta_d$ with Aitchison distance;
\ding{196} $\Delta_d$ with Hilbert distance;
\ding{197} $\Delta_d$ with Funk distance.
For each configuration, we generate $10$ different instances of the random
points/graphs, and the standard deviation is shown as color bands.
We observe that in general, as $d$ grows, all manifolds
have decreasing $\ell(\calD,\calM^d)$.
The jitters and large deviations are due to that the optimizer
stopped at a bad local optimum in some of the experiments.
The Hilbert simplex and the Poincar\'e hyperboloid are observed as the best
geometries which can preserve the input distance matrix.
The Funk distance is asymmetric and is not as good as the other baselines.

Figure~\ref{fig:lossp} shows $\ell(\calP,\calM^d)$ (in log scale)
against $d$ for the investigated geometries.
On the random points dataset, the $L_1$ distance presents an increasing loss with $d$.
This could be due to its mismatch with the geometry of the dataset
and that the optimizer stopped at a local optimum.
Overall, the proposed Hilbert simplex geometry can better represent pairwise
similarities in $\Re^d$ and graph random walk similarity matrix, as compared
with the baselines. Funk geometry also achieves good score in
representing the Erd\H{o}s--R\'enyi graphs.

Figure~\ref{fig:erdos} shows $\ell(\calD,\calM^d)$ (left) and $\ell(\calP,\calM^d)$
(right) against $d$ on the Erd\H{o}s--R\'enyi random graph dataset with $p=0.05$
and $p=0.5$
(in the main text we studied the case when $p=0.2$), where $p$ is the probability
for any pair of nodes $i$ and $j$ to be connected by an edge.

Figure~\ref{fig:barabasi} shows $\ell(\calD,\calM^d)$ (left) and
$\ell(\calP,\calM^d)$ (right) against $d$ on the
Barab\'asi--Albert graphs $G(n,m)$ with $m=1$ and $m=3$
(in the main text we only studied the case when $m=2$),
where $m$ is the number of edges to attach when a new node is created.

In both figures, $\calP$ is random walk similarities on these graph datasets.
We do not simulate real random walks as in graph embedding methods.
Instead, we use the graph adjacency matrix to construct the
transition probability matrix, whose matrix power gives
the random walk similarity matrix $\calP$.

\def\sizefig{0.49\textwidth}

\begin{figure*}[ht]
\begin{subfigure}{\sizefig}
  \includegraphics[width=\textwidth]{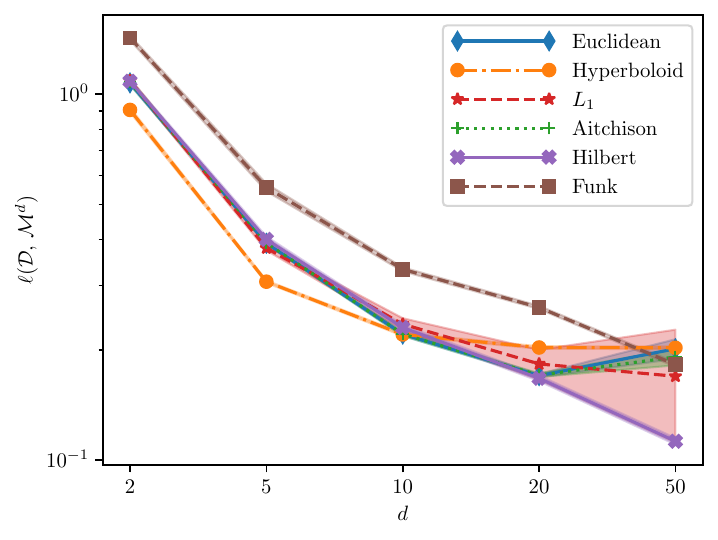}
  \caption{$\ell(\calD,\calM^d)$ ($n=200$, $p=0.05$)}
\end{subfigure}
\begin{subfigure}{\sizefig}
  \includegraphics[width=\textwidth]{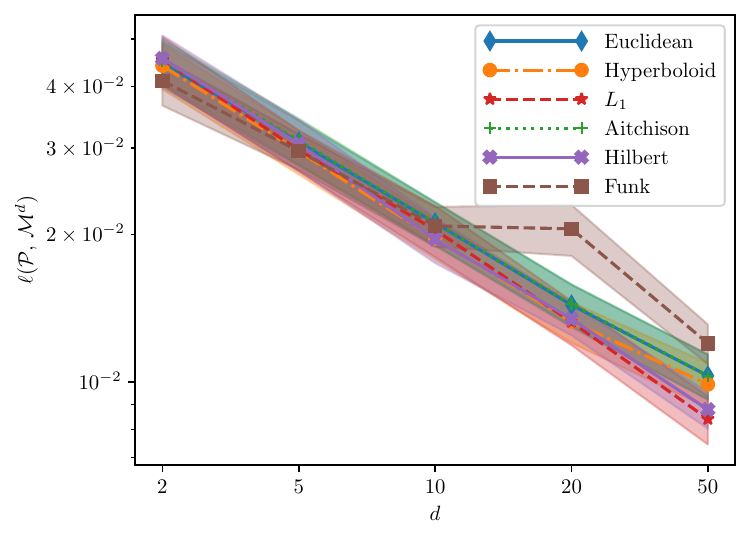}
  \caption{$\ell(\calP,\calM^d)$ ($n=200$, $p=0.05$)}
\end{subfigure}
\begin{subfigure}{\sizefig}
  \includegraphics[width=\textwidth]{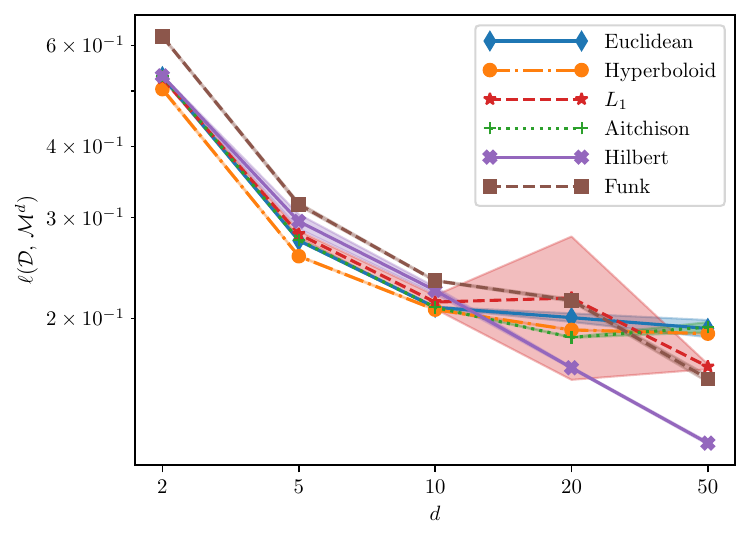}
  \caption{$\ell(\calD,\calM^d)$ ($n=200$, $p=0.5$)}
\end{subfigure}
\begin{subfigure}{\sizefig}
  \includegraphics[width=\textwidth]{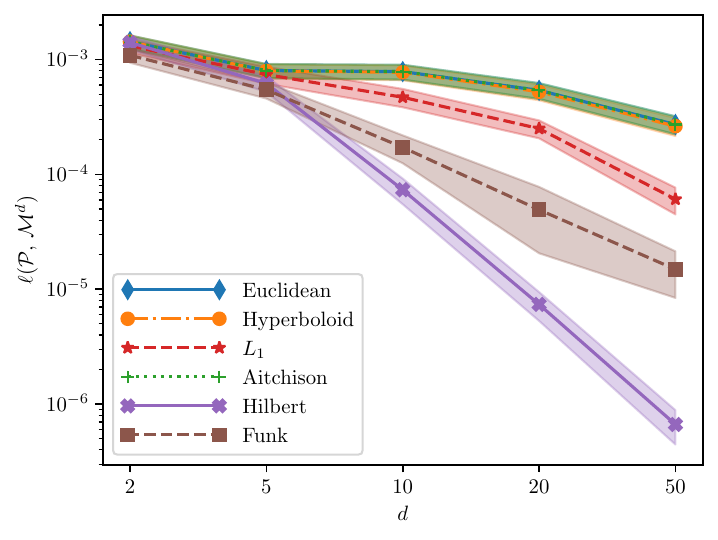}
  \caption{$\ell(\calP,\calM^d)$ ($n=200$, $p=0.5$)}
\end{subfigure}

\caption{Embedding losses against $d$ (Erd\H{o}s--R\'enyi random graph $G(n,p)$).}
\label{fig:erdos}
\end{figure*}

\begin{figure*}[ht]
\begin{subfigure}{\sizefig}
  \includegraphics[width=\textwidth]{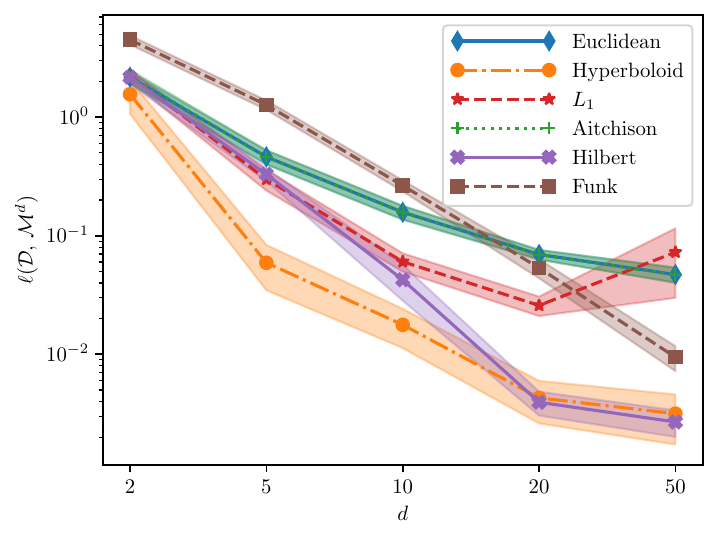}
  \caption{$\ell(\calD,\calM^d)$ ($n=200$, $m=1$)}
\end{subfigure}
\begin{subfigure}{\sizefig}
  \includegraphics[width=\textwidth]{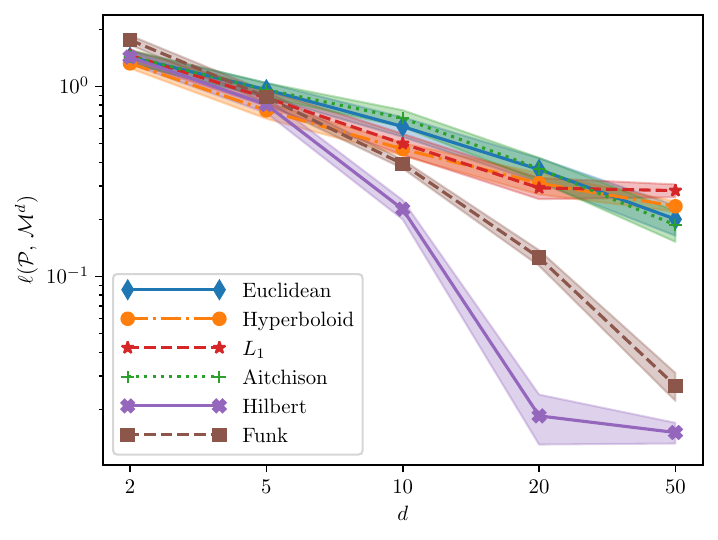}
  \caption{$\ell(\calP,\calM^d)$ ($n=200$, $m=1$)}
\end{subfigure}
\begin{subfigure}{\sizefig}
  \includegraphics[width=\textwidth]{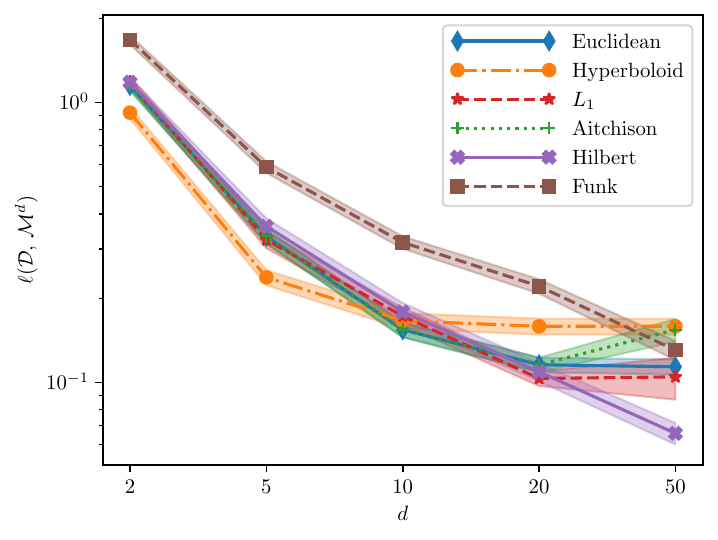}
  \caption{$\ell(\calD,\calM^d)$ ($n=200$, $m=3$)}
\end{subfigure}
\begin{subfigure}{\sizefig}
  \includegraphics[width=\textwidth]{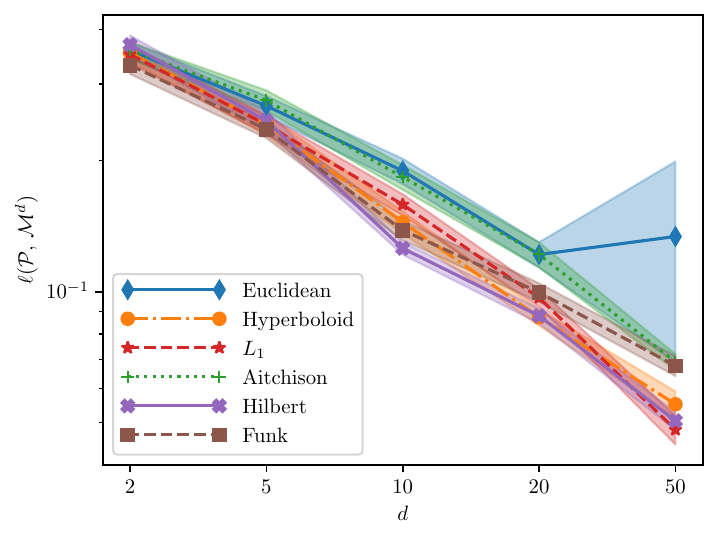}
  \caption{$\ell(\calP,\calM^d)$ ($n=200$, $m=3$)}
\end{subfigure}

\caption{Embedding losses against $d$ (Barab\'asi--Albert graphs $G(n,m)$).}
\label{fig:barabasi}
\end{figure*}

\section{Conclusion and discussion}\label{sec:concl}
We presented the Hilbert simplex
geometry with its closed form distance (Eq.~\ref{eq:rhohd}) and its differentiable approximation (Eq.~\ref{eq:lsetapprox}).
We provided a simple proof
that the Funk and Hilbert 
distances both satisfy the information
monotonicity.
We made use of an isometry
between the Hilbert
simplex and a normed vector space well-suited to carry optimization.
We highlighted a connection between the Aitchison distance and the Hilbert projective distance.
By comparing with commonly-used geometries in machine learning,
we showed experimentally that the Hilbert simplex geometry can better
embed a given distance matrix
or graph random walk similarities.

We showed experimentally that our non-linear embedding technique based on the differentiable approximation of the Hilbert simplex distance
  ($\tilde\rho_\LSET$ of Eq.~\ref{eq:lsetapprox}) is fast, numerically robust, and competitive compared to (hyperboloid) hyperbolic embeddings.
This holds despite the fact that the Hilbert simplex geometry amounts to a normed vector space with ball volumes increasing polynomially with the radii and volume entropy being null~\cite{vernicos2017approximability}.
The efficiency of hyperbolic geometry embeddings relies on the property that ball volumes increases exponentially with radii and  that the volume entropy coincides with the space dimension minus one.
However,  hyperbolic embeddings are prone  to numerical issues in practice~\cite{mishne2023numerical}.
By just slightly and smoothly rounding the simplex, we get a $C^2$ convex boundary finely approximating the simplex boundary which induces a hyperbolic type of Hilbert geometry~\cite{troyanov2013funk,kobayashi2019} with Finsler flag curvature $-1$ 
(generalizing the notion of Riemannian sectional curvature). 
Thus the good experimental results obtained by our differentiable distortion $\tilde\rho_\LSET$ may be explained by the fact that we changed the  underlying nature of the curvature of the space by using the distortion $\tilde\rho_\LSET$ rather than using the non-differentiable Hilbert simplex distance.   

\vskip 0.5cm
\noindent Additional materials are available at~\url{https://franknielsen.github.io/HSG/}

\bibliographystyle{plain}
\bibliography{HSG-embedding}


\end{document}